%% file: gen_bounds.tex
\documentclass{article}

\usepackage{times}
\usepackage{graphicx} % more modern
\usepackage{subfigure} 

\usepackage{natbib}

\usepackage{algorithm}
\usepackage{algorithmic}

\usepackage{hyperref}

\usepackage[accepted]{icml2015}

\icmltitlerunning{Submission and Formatting Instructions for ICML 2015}

\icmltitlerunning{Generalization error bounds for learning to rank}

\usepackage{amssymb,amsmath,amsthm}

\input{macros}

\begin{document} 

\twocolumn[
\icmltitle{Generalization error bounds for learning to rank: \\ 
           Does the length of document lists matter?}

% unless you've provided the [accepted] option to the icml2015
\icmlauthor{Ambuj Tewari}{tewaria@umich.edu}
\icmladdress{University of Michigan, Ann Arbor}
\icmlauthor{Sougata Chaudhuri}{sougata@umich.edu}
\icmladdress{ University of Michigan, Ann Arbor}

\icmlkeywords{learning to rank, generalization error bounds, learning theory, Rademacher complexity}

\vskip 0.2in
]

\input{abstract}
\parskip 4pt
\section{Introduction}
\label{sec:intro}
\input{intro}

\section{Preliminaries}
\label{sec:prelim}
\input{prelim}

\section{Application to Specific Losses}
\label{sec:applications}
\input{applications}

\section{Does The Length of Document Lists Matter?}
\label{sec:length}

Our work is directly motivated by a very interesting generalization bound for learning to rank due to \citet[Theorem 1]{chapelle2010}. They considered a Lipschitz continuous loss $\phi$ with Lipschitz constant $G_\phi^{CW}$ w.r.t. \emph{the $\ell_2$ norm}. They show that, with probability at least $1-\delta$,
\begin{multline*}
\forall w \in \F_2,\ L_\phi(w) \le \hat{L}_\phi(w) + 3 \, G_\phi^{CW} W_2 R_X \sqrt{\frac{m}{n}} \\
+ \sqrt{\frac{8 \log(1/\delta)}{n}} .
\end{multline*}
The dominant term on the right is $O(G_\phi^{CW}W_2 R_X \sqrt{m/n})$. In the next three sections, we will derive improved
bounds of the form $\tilde{O}(G_\phi W_2 R_X \sqrt{1/n})$ where $G_\phi \le G_\phi^{CW} \sqrt{m}$ but can be much smaller.
Before we do that, let us examine the dimensionality reduction in linear scoring function that is caused by a natural permutation invariance requirement.

\subsection{Permutation invariance removes $m$ dependence in dimensionality of linear scoring functions}
As stated in Section~\ref{sec:prelim}, a ranking is obtained by sorting a score vector obtained via a linear scoring function $f$. Consider the space of \emph{linear} scoring function that consists of \emph{all} linear maps $f$ that map $\reals^{m \times d}$ to $\reals^m$:
{\small
\[
\Ffull := \left\{ X \mapsto [\inner{X,W_1}, \ldots, \inner{X,W_m}]^{\top} \::\: W_i \in \reals^{m \times d} \right\} .
\]
}
These linear maps are fully parameterized by matrices $W_1,\ldots,W_m$. Thus, a full parameterization of the linear scoring function is of dimension $m^2 d$.
Note that the popularly used class of linear scoring functions $\Flin$ defined in Eq.~\ref{eq:lindef} is actually a low $d$-dimensional subspace of the full $m^2d$ dimensional space of all linear maps. It is important to note that
the dimension of $\Flin$ is \emph{independent of $m$}.

In learning theory, one of the factors influencing the generalization error bound is the richness of the class of hypothesis functions. Since the linear function class $\Flin$ has dimension independent of $m$, we intuitively expect that, at least
under some conditions, algorithms that minimize ranking losses using linear scoring functions should have an $m$ independent complexity term in the generalization bound.
The reader might wonder whether the dimension reduction from $m^2d$ to $d$ in going from $\Ffull$ to $\Flin$ is arbitrary. To dispel this doubt, we prove the lower dimensional class $\Flin$ is the \emph{only sensible choice} of linear scoring functions
in the learning to rank setting. This is because scoring functions should satisfy a permutation invariance property. That is, if we apply a permutation $\pi \in S_m$ to the rows of $X$ to get a matrix $\pi X$ then the scores should also simply get permuted
by $\pi$. That is, we should only consider scoring functions in the following class:
\[
\Fperm = \{ f : \forall \pi \in S_m, \forall X \in \reals^{m \times d}, \pi f(X) = f(\pi X) \}.
\]
The permutation invariance requirement, in turn, forces a reduction from dimension $m^2d$ to just $2d$ (which has no dependence on $m$).

\begin{theorem}
\label{thm:dimension}
The intersection of the function classes $\Ffull$ and $\Fperm$ is the $2d$-dimensional class:
\begin{equation}\label{eq:reduced}
\Flin' = \{ X \mapsto Xw + (\one^\top X v) \one \::\: w,v \in \reals^d \} .
\end{equation}
\end{theorem}

Note that the extra degree of freedom provided by the $v$ parameter in Eq.~\ref{eq:reduced} is useless for ranking purposes since adding a constant vector (i.e., a multiple of $\one$) to a score vector has no effect on the sorted order.
This is why we said that $\Flin$ is the only sensible choice of linear scoring functions.

\section{Online to Batch Conversion}
\label{sec:online}

In this section, we build some intuition as to why it is natural to use $\|\cdot\|_\infty$ in defining the Lipschitz constant of the loss $\phi$.
To this end, consider the following well known online gradient descent (OGD) regret guarantee. Recall that OGD refers to the simple online 
algorithm that makes the update $w_{i+1} \gets w_i - \eta \nabla_{w_i} f_i(w_i)$ at time $i$. If we run OGD to generate $w_i$'s, we have,
for all $\|w\|_2 \le W_2$: 
\[
\sum_{i=1}^n f_i(w_i) - \sum_{i=1}^n f_i(w) \le \frac{W_2^2}{2\eta} + \eta G^2 n
\]
where $G$ is a bound on the maximum $\ell_2$-norm of the gradients $\nabla_{w_i} f_i(w_i)$ and $f_i$'s have to be \emph{convex}.
If $(X^{(1)},y^{(1)}),\ldots,
(X^{(n)},y^{(n)})$ are iid then by setting $f_i(w) = \phi(X^{(i)}w,y^{(i)})$, $1\le i \le n$ we can do an ``online to batch conversion". That is, we optimize over $\eta$, take expectations and use Jensen's inequality to get the following excess risk bound:
\[
\forall \|w\|_2 \le W_2,\ \E{L_\phi(\hat{w}_{\mathrm{OGD}})} - L_\phi(w) \le W_2 G \sqrt{\frac{2}{n}}  
\]
where $\hat{w}_{\mathrm{OGD}} = \tfrac{1}{n} \sum_{i=1}^n w_i$ and $G$ has to satisfy (noting that $s= X^{(i)} w_i$)
\[
G \ge \| \nabla_{w_i} f_i(w_i) \|_2 = \| (X^{(i)})^\top \nabla_s \phi(X^{(i)}w_i, y^{(i)}) \|_2
\]
where we use the chain rule to express $\nabla_w$ in terms of $\nabla_s$.
Finally, we can upper bound 
\begin{align*}
&\quad \| (X^{(i)})^\top \nabla_s \phi(X^{(i)}w_i, y^{(i)}) \|_2 \ \\
& \le \| (X^{(i)})^\top \|_{1 \to 2} \cdot \| \nabla_s \phi(X^{(i)}w_i, y^{(i)}) \|_1 \\
& \le R_X  \| \nabla_s \phi(X^{(i)}w_i, y^{(i)}) \|_1 
\end{align*}
as $R_X \ge \max_{j=1}^m \| X_j \|_2$ and because of the following lemma.
%&= \max_{j=1}^m \| X_j \|_2  \cdot \| \nabla_s \phi(X^{(i)}w_i, y^{(i)}) \|_1

\begin{lemma}
\label{lem:normexpr}
For any $1 \leq p \leq \infty$,
%\[
%\|X\|_{p \to q}= \underset{v \neq 0}{\sup} \frac{\|Xv\|_q}{\|v\|_p} \text{and}\ \\
%
% \| X^\top \|_{1 \to p} = \| X \|_{q \to \infty} = \max_{j=1}^m \| X_j \|_p \ ,
%\]
\begin{equation*}
\begin{split}
&\|X\|_{p \to q}= \underset{v \neq 0}{\sup} \frac{\|Xv\|_q}{\|v\|_p} \\
& \| X^\top \|_{1 \to p} = \| X \|_{q \to \infty} = \max_{j=1}^m \| X_j \|_p \ ,
\end{split}
\end{equation*}
where $q$ is the dual exponent of $p$ (i.e., $\tfrac{1}{q}+\tfrac{1}{p} = 1$).
\end{lemma}

Thus, we have shown the following result.

\begin{theorem}
\label{thm:online}
Let $\phi$ be convex and have Lipschitz constant $G_\phi$ w.r.t. $\|\cdot\|_\infty$. Suppose we run online gradient descent (with appropriate step size $\eta$)
on $f_i(w) = \phi(X^{(i)}w,y^{(i)})$ and return $\hat{w}_{\mathrm{OGD}} = \tfrac{1}{T} \sum_{i=1}^n w_i$. Then we have,
\[
\forall \|w\|_2 \le W_2,\ \E{L_\phi(\hat{w}_{\mathrm{OGD}})} - L_\phi(w) \le G_\phi\,W_2\,R_X \sqrt{\frac{2}{n}} .
\]
\end{theorem}
The above excess risk bound has no explicit $m$ dependence. This is encouraging but there are two deficiencies of this approach based
on online regret bounds.  First, the result applies to the output of a specific algorithm that may not be the method of choice
for practitioners. For example, the above argument does not yield uniform convergence bounds that could lead to excess risk bounds for ERM (or regularized versions of it).
Second, there is no way to generalize the result to Lipschitz, but \emph{non-convex} loss functions. It may noted here that the original motivation
for \citet{chapelle2010} to prove their generalization bound was to consider the non-convex loss used in their SmoothRank method.
We will address these issues in the next two sections.

\section{Stochastic Convex Optimization}
\label{sec:sco}

We first define the regularized empirical risk minimizer:
\begin{equation}
\label{eq:minimizer}
\hat{w}_\lambda = \argmin_{\|w\|_2 \le W_2}\ \frac{\lambda}{2 }\|w\|_{2}^2 + \hat{L}_\phi(w) .
\end{equation}

We now state the main result of this section.
\begin{theorem}
\label{thm:sco}
Let the loss function $\phi$ be convex and have Lipschitz constant $G_\phi$ w.r.t. $\|\cdot\|_\infty$. Then, for an appropriate choice of $\lambda=O(1/\sqrt{n})$, we have
\begin{equation*}
\E{L_\phi(\hat{w}_\lambda) } \le  L_\phi(w^\star)  + 2\, G_\phi\, R_X\, W_2\,\left(\frac{8}{n} + \sqrt{\frac{2}{n}}\right) .
\end{equation*}
\end{theorem}

This result applies to a batch algorithm (regularized ERM) but unfortunately requires the regularization parameter $\lambda$ to be set in a particular way.
Also, it does not apply to non-convex losses and does not yield uniform convergence bounds. In the next section, we will address these deficiencies. However, we will incur some extra logarithmic factors that are
absent in the clean bound above.

\section{Bounds for Non-convex Losses}
\label{sec:covering}

The above discussion suggests that we have a possibility of deriving tighter, possibly $m$-independent, generalization error bounds by assuming
that $\phi$ is Lipschitz continuous w.r.t. $\|\cdot\|_\infty$. The standard approach in binary classification is to appeal to the Ledoux-Talagrand 
contraction principle for establishing Rademacher complexity \citep{bartlett2003rademacher}. It gets rid of the loss function and incurs a factor equal to the Lipschitz constant of the loss in the Rademacher complexity bound. Since the loss function takes scalar argument, the Lipschitz constant is defined for only one norm, i.e., the absolute value norm.
It is not immediately clear how such an approach would work when the loss takes vector valued arguments and is Lipschitz w.r.t. $\|\cdot\|_\infty$.
We are not aware of an appropriate extension of the Ledoux-Talagrand contraction principle. Note that Lipschitz continuity w.r.t. the Euclidean
norm $\|\cdot\|_2$ does not pose a significant challenge since Slepian's lemma can be applied to get rid of the loss. Several authors have already exploited Slepian's lemma in this context \citep{bartlett2003rademacher,chapelle2010}.
%In the absence of a general principle that would allow us to deal with an arbitrary loss that is Lipschitz w.r.t. $\|\cdot\|_\infty$,
We take a route
involving covering numbers and define the data-dependent (pseudo-)metric:
\[
d_{\infty}^{Z^{(1:n)}}(w,w') := \max_{i=1}^n \left| \phi(X^{(i)}w,y^{(i)}) - \phi(X^{(i)}w',y^{(i)}) \right|
\]
Let $\N_\infty(\epsilon,\phi \circ \F,Z^{(1:n)})$ be the covering number at scale $\epsilon$ of the composite class $\phi \circ \F = \phi \circ \F_1$ or $ \phi \circ\F_2$ w.r.t. the above metric. Also 
define
$$\N_\infty(\epsilon,\phi \circ \F,n) := \max_{Z^{(1:n)}} \N_\infty(\epsilon,\phi \circ \F,Z^{(1:n)}) .$$
With these definitions in place, we can state our first result on covering numbers.
\begin{proposition}
\label{prop:Ninftycover}
Let the loss $\phi$ be Lipschitz w.r.t. $\|\cdot\|_\infty$ with constant $G_\phi$. Then following covering number bound holds:
\begin{align*}
\log_2 \N_\infty(\epsilon,\phi \circ \F_2,n) &\le \left\lceil \frac{G_\phi^2 \, W_2^2 \, R_X^2}{\epsilon^2} \right\rceil \log_2 (2mn+1) .
\end{align*}
\end{proposition}
\begin{proof}
Note that
\begin{align*}
&\quad \max_{i=1}^n \left| \phi(X^{(i)}w,y^{(i)}) - \phi(X^{(i)}w',y^{(i)}) \right| \\
&\le G_\phi \cdot \max_{i=1}^n \max_{j=1}^m \left| \inner{X^{(i)}_j,w} - \inner{ X^{(i)}_j, w'} \right| .
\end{align*}
This immediately implies that if we have a cover of the class $\G_2$ (Sec.\ref{sec:prelim}) 
at scale $\epsilon/G_\phi$ w.r.t. the metric
\[
\max_{i=1}^n \max_{j=1}^m \left| \inner{X^{(i)}_j,w} - \inner{ X^{(i)}_j, w'} \right|
\]
then it is also a cover of $\phi \circ \F_2$
w.r.t. $d_{\infty}^{Z^{(1:n)}}$, at scale $\epsilon$. Now comes a simple, but crucial observation: \emph{from the point of view of the scalar
valued function class $\G_2$, the vectors $(X^{(i)}_j)_{j=1:m}^{ i=1:n}$ constitute a data set of size $mn$}. Therefore,
\begin{equation}
\label{eq:coveringofFfromG2}
\N_\infty(\epsilon,\phi \circ \F_2,n) \le \N_\infty(\epsilon/G_\phi,\G_2,mn) .
\end{equation}
Now we appeal to the following bound due  to \citet[Corollary 3]{zhang2002covering} (and plug the result into~\eqref{eq:coveringofFfromG2}):
\[
\log_2 \N_\infty(\epsilon/G_\phi,\G_2,mn) \le \left\lceil \frac{G_\phi^2 \, W_2^2 \, R_X^2}{\epsilon^2} \right\rceil \log_2 (2mn+1)
\]

\end{proof}
Covering number $\N_2(\epsilon,\phi \circ \F,Z^{(1:n)})$ uses pseudo-metric:
{\small\[
d_{2}^{Z^{(1:n)}}(w,w') := \left( \sum_{i=1}^n \frac{1}{n} \left( \phi(X^{(i)}w,y^{(i)}) - \phi(X^{(i)}w',y^{(i)}) \right)^2 \right)^{1/2} 
\]}
It is well known that a control on
$\N_2(\epsilon,\phi \circ \F,Z^{(1:n)})$ provides control on the empirical Rademacher complexity and that $\N_2$ covering numbers are smaller than $\N_\infty$ ones. For us, it will be convenient to use
a more refined version\footnote{We use a further refinement due to
Srebro and Sridharan available at \url{http://ttic.uchicago.edu/~karthik/dudley.pdf}} due to \citet{mendelson2002rademacher}.
Let $\mathcal{H}$ be a class of functions, with $\mathcal{H}: \Z \mapsto \mathbb{R}$, uniformly bounded by $B$. Then, we have following bound on empirical Rademacher complexity
{\small
\begin{align}
&\quad \radhat{n}{\mathcal{H}} \notag\\
\label{eq:dudley1}
&\le \inf_{\alpha > 0} \left( 4 \alpha + 10 \int_{\alpha}^{\sup_{h\in \mathcal{H}} \sqrt{\Ehat{h^2}}} \sqrt{\frac{\log_2 \N_2(\epsilon,\mathcal{H},Z^{(1:n)})}{n}} d\epsilon \right) \\
\label{eq:dudley2}
&\le \inf_{\alpha > 0} \left( 4 \alpha + 10 \int_{\alpha}^{B} \sqrt{\frac{\log_2 \N_2(\epsilon,\mathcal{H},Z^{(1:n)})}{n}} d\epsilon \right) .
\end{align}
}
Here $\radhat{n}{\mathcal{H}}$ is the empirical Rademacher complexity of the class $\mathcal{H}$ defined as
\begin{equation*}
\radhat{n}{\mathcal{H}} := \Es{\sigma_{1:n}}{ \sup_{h \in \mathcal{H}} \frac{1}{n} \sum_{i=1}^n \sigma_i h(Z_i) } ,
\end{equation*}
where $\sigma_{1:n} = (\sigma_1,\ldots,\sigma_n)$ are iid Rademacher (symmetric Bernoulli) random variables. 

\begin{corollary}
Let $\phi$ be Lipschitz w.r.t. $\|\cdot\|_\infty$ and uniformly bounded\footnote{A uniform bound on the loss easily follows under the (very reasonable) assumption that
$\forall y, \exists s_y \text{ s.t. } \phi(s_y,y) = 0$. Then $\phi(Xw,y) \le G_\phi\|Xw - s_y\|_\infty \le G_\phi(W_2 R_X + \max_{y\in\Y} \|s_y\|_\infty) \le G_\phi(2W_2 R_X) $.}
by $B$ for $w\in \F_2$.
Then the empirical Rademacher complexities of the class $\phi \circ \F_2$ is bounded as
\begin{align*}
\radhat{n}{\phi \circ \F_2} &\le 10 G_\phi W_2 R_X \sqrt{\frac{\log_2(3mn)}{n}} \\
&\times \log \tfrac{6B\sqrt{n}}{5G_\phi W_2 R_X \sqrt{\log_2(3mn)}} .
\end{align*}
\end{corollary}
\begin{proof}
This follows by simply plugging in estimates from Proposition~\ref{prop:Ninftycover} into~\eqref{eq:dudley2} and choosing $\alpha$ optimally.
\end{proof}

Control on the Rademacher complexity immediately leads to uniform convergence bounds and generalization error bounds for ERM. The informal
$\tilde{O}$ notation hides factors logarithmic in $m,n,B,G_\phi,R_X,W_1$. Note that all hidden factors are small and computable from the results 
above.

\begin{theorem}
\label{thm:lip}
Suppose $\phi$ is Lipschitz w.r.t. $\|\cdot\|_\infty$ with constant $G_\phi$ and is uniformly bounded by $B$ as $w$ varies over $\F_2$.
With probability at least $1-\delta$,
\begin{multline*}
\forall w \in \F_2,\ L_\phi(w) \le \hat{L}_\phi(w) \\
+ \tilde{O}\left( G_\phi W_2 R_X \sqrt{\frac{1}{n}} + B \sqrt{ \frac{\log(1/\delta)}{n} } \right)
\end{multline*}
and therefore with probability at least $1-2\delta$,
\[
L_\phi(\hat{w}) \le L_\phi(w^\star) + \tilde{O}\left( G_\phi W_2 R_X \sqrt{\frac{1}{n}} + B \sqrt{ \frac{\log(1/\delta)}{n} } \right) .
\]
where $\hat{w}$ is an empirical risk minimizer over $\F_2$.
\end{theorem}
\begin{proof}
Follows from standard bounds using Rademacher complexity. See, for example, \citet{bartlett2003rademacher}.
\end{proof}
As we said before, ignoring logarithmic factors, the bound for $\F_2$ is an improvement over the bound of \citet{chapelle2010}.

\section{Extensions}
\label{sec:ext}

We extend the generalization bounds above to two settings: a) high dimensional features and b) smooth losses.

\subsection{High-dimensional features}

In learning to rank situations involving high dimensional features, it may not be appropriate to use the class $\F_2$ of $\ell_2$ bounded
predictors. Instead, we would like to consider the class $\F_1$ of $\ell_1$ bounded predictors. In this case, it is natural
to measure size of the input matrix $X$ in terms of a bound $\bar{R}_X$ on the maximum $\ell_\infty$ norm of each of its row.
The following analogue of Proposition~\ref{prop:Ninftycover} can be shown.

\begin{proposition}
\label{prop:Ninftycover1}
Let the loss $\phi$ be Lipschitz w.r.t. $\|\cdot\|_\infty$ with constant $G_\phi$. Then the following covering number bound holds:
\begin{align*}
\log_2 \N_\infty(\epsilon,\phi \circ \F_1,n) &\le \left\lceil \frac{288\,G_\phi^2 \, W_1^2 \, \bar{R}_X^2\,(2+\log d)}{\epsilon^2} \right\rceil \\
&\times \log_2 \left( 2 \left\lceil \frac{8 G_\phi W_1 \bar{R}_X}{\epsilon} \right\rceil mn+1 \right)  .
\end{align*}
\end{proposition}
Using the above result to control the Rademacher complexity of $\phi \circ \F_1$ gives the following bound.
\begin{corollary}
Let $\phi$ be Lipschitz w.r.t. $\|\cdot\|_\infty$ and uniformly bounded by $B$ for $w\in \F_1$.
Then the empirical Rademacher complexities of the class $\phi \circ \F_1$ is bounded as
{\small
\begin{align*}
\radhat{n}{\phi \circ \F_1} &\le 120\sqrt{2} G_\phi W_1 \bar{R}_X \sqrt{\frac{\log( d) \, \log_2(24mnG_\phi W_1 \bar{R}_X) }{n}} \\
&\quad \quad \times \log^2 \tfrac{B + 24mnG_\phi W_1 \bar{R}_X }{40\sqrt{2} G_\phi W_1 \bar{R}_X  \sqrt{\log( d) \, \log_2(24mnG_\phi W_1 \bar{R}_X) } } .
\end{align*}
}
\end{corollary}
As in the previous section, control of Rademacher complexity immediately yields uniform convergence and ERM generalization error bounds.
\begin{theorem}
\label{thm:lip1}
Suppose $\phi$ is Lipschitz w.r.t. $\|\cdot\|_\infty$ with constant $G_\phi$ and is uniformly bounded by $B$ as $w$ varies over $\F_1$.
With probability at least $1-\delta$,
\begin{multline*}
\forall w \in \F_1,\ L_\phi(w) \le \hat{L}_\phi(w) \\
+ \tilde{O}\left( G_\phi W_1 \bar{R}_X  \sqrt{\frac{\log d}{n}} + B \sqrt{ \frac{\log(1/\delta)}{n} } \right)
\end{multline*}
and therefore with probability at least $1-2\delta$,
\[
L_\phi(\hat{w}) \le L_\phi(w^\star) + \tilde{O}\left( G_\phi W_1 \bar{R}_X \sqrt{\frac{\log d}{n}} + B \sqrt{ \frac{\log(1/\delta)}{n} } \right) 
\]
where $\hat{w}$ is an empirical risk minimizer over $\F_1$.
\end{theorem}
As can be easily seen from Theorem.~\ref{thm:lip1}, the generalization bound is \emph{almost} independent of the dimension of the document feature vectors. We are not aware of existence of such a result in learning to rank literature.
\subsection{Smooth losses}

We will again use online regret bounds to explain why we should expect ``optimistic" rates for smooth losses before giving
more general results for smooth but possibly non-convex losses.

\subsection{Online regret bounds under smoothness}
\label{sec:onlinesmooth}

Let us go back to OGD guarantee, this time presented in a slightly more refined version. If we run OGD with learning rate $\eta$ then, for all $\|w\|_2 \le W_2$:
\[
\sum_{i=1}^n f_i(w_i) - \sum_{i=1}^n f_i(w) \le \frac{W_2^2}{2\eta} + \eta \sum_{i=1}^n \| g_i \|_2^2
\]
where $g_i = \nabla_{w_i}f_i(w_i)$ (if $f_i$ is not differentiable at $w_i$ then we can set $g_i$ to be an arbitrary subgradient of $f_i$ at $w_i$). Now
assume that all $f_i$'s are non-negative functions and are smooth w.r.t. $\|\cdot\|_2$ with constant $H$. Lemma 3.1 of \citet{srebro2010smoothness}
tells us that any non-negative, smooth function $f(w)$ enjoy an important \emph{self-bounding} property
for the gradient:
\[
\| \nabla_w f_i(w) \|_2 \le \sqrt{ 4 H f_i(w) }
\]
which bounds the magnitude of the gradient of $f$ at a point in terms of the value of the function itself at that point. This means that
$\|g_i\|_2^2 \le 4Hf_i(w_i)$ which, when plugged into the OGD guarantee, gives:
\[
\sum_{i=1}^n f_i(w_i) - \sum_{i=1}^n f_i(w) \le \frac{W_2^2}{2\eta} + 4\eta H \sum_{i=1}^n f_i(w_i)
\]
Again, setting $f_i(w) = \phi(X^{(i)}w,y^{(i)})$, $1\le t \le n$, and using the online to batch conversion technique, we can arrive at the bound:
for all $\|w\|_2 \le W_2$:
\[
\E{L_\phi(\hat{w})} \le \frac{L_\phi(w)}{(1-4\eta H)} + \frac{W_2^2}{2\eta(1-4\eta H)n}
\]
At this stage, we can fix $w=w^\star$, the optimal $\ell_2$-norm bounded predictor and get optimal $\eta$ as:
\begin{equation}
\label{eq:etaval}
\eta = \frac{W_2}{4HW_2 + 2\sqrt{ 4H^2W_2^2 + 2H L_\phi(w^\star) n }} .
\end{equation}
After plugging this value of $\eta$ in the bound above and some algebra (see Section~\ref{app:etacalc}), we get the upper bound
\begin{equation}
\label{eq:onlinefastrate}
\E{L_\phi(\hat{w})} \le L_\phi(w^\star) + 2\sqrt{ \frac{2HW_2^2 L_\phi(w^\star)}{n} } + \frac{8HW_2^2}{n} .
\end{equation}
Such a rate interpolates between a $1/\sqrt{n}$ rate in the ``pessimistic" case ($L_\phi(w^\star) > 0$) and the $1/n$ rate in the
 ``optimistic" case ($L_\phi(w^\star) = 0$) (this terminology is due to \citet{panchenko2002some}).

Now, assuming $\phi$ to be twice differentiable, we need $H$ such that
{\small
\[
H \ge \| \nabla_w^2 \phi(X^{(i)}w, y^{(i)}) \|_{2\to 2} = \| X^\top \nabla_s^2 \phi(X^{(i)}w, y^{(i)}) X \|_{2 \to 2}
\]
}
where we used the chain rule to express $\nabla_w^2$ in terms of $\nabla_s^2$. Note that, for OGD, we need smoothness in $w$ w.r.t. $\|\cdot\|_2$
which is why the matrix norm above is the operator norm corresponding to the pair $\|\cdot\|_2,\|\cdot\|_2$. In fact, when we say ``operator norm"
without mentioning the pair of norms involved, it is this norm that is usually meant. It is well known that this norm is equal to the largest singular value 
of the matrix. But, just as before, we can bound this in terms of the smoothness  constant of $\phi$ w.r.t. $\|\cdot\|_\infty$ (see Section~\ref{app:smoothness} in the appendix):
\begin{align*}
&\quad \| (X^{(i)})^\top \nabla_s^2 \phi(X^{(i)}w, y^{(i)}) X^{(i)} \|_{2 \to 2} \\
&\le R_X^2 \| \nabla_s^2 \phi(X^{(i)}w, y^{(i)}) \|_{\infty \to 1} .
\end{align*}

where we used Lemma~\ref{lem:normexpr} once again. 
This result using online regret bounds is great for building intuition but suffers from the two defects we mentioned at the end of Section~\ref{sec:online}. 
In the smoothness case, it additionally suffers from a more serious defect: the correct choice of the learning rate $\eta$
requires knowledge of $L_\phi(w^\star)$
which is seldom available.

\subsection{Generalization error bounds under smoothness}
\label{sec:gensmooth}

Once again, to prove a general result for possibly non-convex smooth losses, we will adopt an approach based on covering numbers.
To begin, we will need a useful lemma from \citet[Lemma A.1 in the Supplementary Material]{srebro2010smoothness}. Note that,
for functions over real valued predictions, we do not need to talk about the norm when dealing with smoothness since essentially the only norm available is the
absolute value.

\begin{lemma}
\label{lem:scalarsmooth}
For any $h$-smooth non-negative function $f: \reals \to \reals_+$ and any $t,r \in \reals$	we have
\[
(f(t) - f(r))^2 \le 6 h (f(t) + f(r)) (t-r)^2 .
\]
\end{lemma}
We first provide an extension of this lemma to the vector case.
\begin{lemma}
\label{lem:vectorsmooth}
If $\phi: \reals^m \to \reals_+$ is a non-negative function with smoothness constant $H_\phi$ w.r.t. a norm $\lipl \cdot \lipr$ then
for any $s_1,s_2  \in \reals^m$ we have
\[
(\phi(s_1) - \phi(s_2))^2 \le 6 H_\phi \cdot (\phi(s_1) + \phi(s_2)) \cdot \lipl s_1 - s_2 \lipr^2 .
\]
\end{lemma}
%\begin{proof}
%See Appendix~\ref{app:vectorsmooth}.
%\end{proof}

Using the basic idea behind local Rademacher complexity analysis, we define the following loss class:
\[
\F_{\phi,2}(r) := \{ (X,y) \mapsto \phi(Xw,y) \::\: \|w\|_2 \le W_2, \hat{L}_\phi(w) \le r \} .
\]
%\[
%\F_{\phi,1}(r) := \{ w \in \F_1 \::\: \hat{L}_\phi(w) \le r \} .
%\]
Note that this is a random subclass of functions since $\hat{L}_\phi(w)$ is a random variable.

\begin{proposition}
\label{prop:N2cover}
Let $\phi$ be smooth w.r.t. $\|\cdot\|_\infty$ with constant $H_\phi$.
The covering numbers of $\F_{\phi,2}(r)$ in the $d_{2}^{Z^{(1:n)}}$ metric defined above are bounded as follows:
{\small
\[
\log_2 \N_2(\epsilon,\F_{\phi,2}(r),Z^{(1:n)}) \le \left\lceil \frac{12H_\phi \, W_2^2 \, R_X^2 \,r}{\epsilon^2} \right\rceil \log_2 (2mn+1) .
\]
}
%\[
%\log_2 \N_2(\epsilon,\F_{\phi,1}(r),Z^{(1:n)}) \le
%\]
\end{proposition}
Control of covering numbers easily gives a control on the Rademacher complexity of the random subclass $\F_{\phi,2}(r)$.
\begin{corollary}
\label{cor:subroot}
Let $\phi$ be smooth w.r.t. $\|\cdot\|_\infty$ with constant $H_{\phi}$ and uniformly bounded by $B$ for $w\in \F_2$. Then the
empirical Rademacher complexity of the class $\F_{\phi,2}(r)$ is bounded as
\[
\radhat{n}{\F_{\phi,2}(r)} \le 4\sqrt{r} C \log \frac{3\sqrt{B}}{C} 
\] 
where $C = 5\sqrt{3}W_2R_X\sqrt{\frac{H_\phi \log_2(3mn)}{n}}$.
\end{corollary}
%\begin{proof}
%See Appendix~\ref{app:subroot}.
%\end{proof}

With the above corollary in place we can now prove our second key result.
\begin{theorem}
\label{thm:smooth}
Suppose $\phi$ is smooth w.r.t. $\|\cdot\|_\infty$ with constant $H_\phi$ and is uniformly bounded by $B$ over $\F_2$.
With probability at least $1-\delta$,
\[
\forall w \in \F_2,\ L_\phi(w) \le \hat{L}_\phi(w) + \tilde{O}\left(
\sqrt{ \frac{L_\phi(w) D_0}{n}}  + \frac{D_0}{n} 
\right)
\]
where $D_0 = B\log(1/\delta) + W_2^2 R_X^2 H_\phi$.
Moreover, with probability at least $1-2\delta$,
\[
L_\phi(\hat{w}) \le L_\phi(w^\star) + \tilde{O}\left(
\sqrt{ \frac{L_\phi(w^\star) D_0}{n}}  + \frac{D_0}{n} 
\right)
\]
where $\hat{w}, w^\star$ are minimizers of $\hat{L}_\phi(w)$ and $L_\phi(w)$ respectively (over $w \in \F_2$).
\end{theorem}

\section{Conclusion}
\label{sec:conclusion}
\input{conclusion}

\section*{Acknowledgement}
We gratefully acknowledge the support of NSF under grant IIS-1319810.
Thanks to Prateek Jain for discussions that led us to Theorem 3.


\bibliography{gen_bounds_bib}
\bibliographystyle{icml2015}

\onecolumn % switch one column for appendix
\newpage

\appendix

\section{Proof of Proposition~\ref{prop:listnet}}
\label{app:listnet}
\begin{proof}
Let $e_j$'s denote standard basis vectors. We have
\[
\nabla_s \listnet(s,y) = -\sum_{j=1}^m P_j(y) e_j + \sum_{j=1}^m  \frac{\exp(s_j)}{\sum_{j'=1}^m \exp(s_{j'})} e_j
\]
Therefore,
\begin{align*}
\| \nabla_s \listnet(s,y) \|_1 &\le \sum_{j=1}^m P_j(y) \|e_j\|_1 + \sum_{j=1}^m  \frac{\exp(s_j)}{\sum_{j'=1}^m \exp(s_{j'})} \| e_j \|_1 \\
& =2 .
\end{align*}
We also have
\begin{equation*}
[\nabla^2_s \listnet(s,y)]_{j,k} = 
\begin{cases}
- \frac{\exp(2s_j)}{(\sum_{j'=1}^m \exp(s_{j'}) )^2} + \frac{\exp(s_j)}{\sum_{j'=1}^m \exp(s_{j'})} & \text{if } j=k \\
- \frac{\exp(s_j+s_k)}{(\sum_{j'=1}^m \exp(s_{j'}) )^2} & \text{if } j \neq k\ .
\end{cases}
\end{equation*}
Moreover,
\begin{align*}
\| \nabla^2_s \listnet(s,y) \|_{\infty \to 1} &\le \sum_{j=1}^m \sum_{k=1}^m | [\nabla^2_s \listnet(s,y)]_{j,k} |\\
&\le \sum_{j=1}^m \sum_{k = 1}^m \frac{\exp(s_j+s_k)}{(\sum_{j'=1}^m \exp(s_{j'}) )^2} 
+ \sum_{j=1}^m \frac{\exp(s_j)}{\sum_{j'=1}^m \exp(s_{j'})} \\
&= \frac{(\sum_{j=1}^m \exp(s_{j}) )^2}{(\sum_{j'=1}^m \exp(s_{j'}) )^2} + \frac{\sum_{j=1}^m \exp(s_{j})}{\sum_{j'=1}^m \exp(s_{j'})} \\
&= 2
\end{align*}
\end{proof}

\section{Proof of Proposition~\ref{prop:sdcg1}}
\label{app:sdcg1}
\begin{proof}
Let $1_{(\text{condition})}$ denote an indicator variable. We have
\[
[\nabla_s \sdcg(s,y)]_j = D(1) \left( \sum_{i=1}^m G(r_i) \left[ \frac{1}{\sigma} \frac{\exp(s_i/\sigma)}{\sum_{j'} \exp(s_{j'}/\sigma)} 1_{(i = j)} 
- \frac{1}{\sigma}\frac{\exp((s_i+s_j)/\sigma)}{(\sum_{j'} \exp(s_{j'}/\sigma))^2}  \right] \right)
\]
Therefore,
\begin{align*}
\frac{\| \nabla_s \sdcg(s,y) \|_1}{D(1)G(Y_{\max})} &\le \sum_{j=1}^m \left(
\sum_{i=1}^m \left[ \frac{1}{\sigma} \frac{\exp(s_i/\sigma)}{\sum_{j'} \exp(s_{j'}/\sigma)} 1_{(i = j)} 
+ \frac{1}{\sigma}\frac{\exp((s_i+s_j)/\sigma)}{(\sum_{j'} \exp(s_{j'}/\sigma))^2}  \right] 
\right) \\
& = \frac{1}{\sigma} \left( \frac{\sum_{j} \exp(s_{j}/\sigma)}{\sum_{j'} \exp(s_{j'}/\sigma)} + \frac{(\sum_{j} \exp(s_{j}/\sigma))^2}{(\sum_{j'} \exp(s_{j'}/\sigma))^2} \right) \\
& = \frac{2}{\sigma} .
\end{align*}
\end{proof}

\section{RankSVM}

The RankSVM surrogate is defined as:

\begin{equation*}
\phi_{RS}(s,y)=\sum_{i=1}^m \sum_{j=1}^m \max(0,1_{(y_i>y_j)}(1+ s_j- s_i)) 
\end{equation*}

It is easy to see that $\nabla_s \phi_{RS} (s,y)=  \sum_{i=1} ^ {m} \sum_{j=1}^m \max (0,1_{(y_i>y_j)}(1 + s_j - s_i)) (e_j - e_i)$. Thus, the $\ell_1$ norm of gradient  is $O(m^2)$ .

\section{Proof of Theorem~\ref{thm:dimension}}
\begin{proof}
It is straightforward to check that $\Flin'$ is contained in both $\Ffull$ as well as $\Fperm$. So, we just need to prove that any $f$ that is in both $\Ffull$ and $\Fperm$ has to be in $\Flin'$ as well.

Let $P_\pi$ denote the $m \times m$ permutation matrix corresponding to a permutation $\pi$. Consider the full linear class $\Ffull$. In matrix notation, the permutation invariance property means that, for any $\pi, X$, we have $P_{\pi}[\inner{X,W_1}, \ldots, \inner{X,W_m}\rangle]^\top= [\inner{P_\pi X,W_1}, \ldots, \inner{P_\pi X,W_m}]^\top$.

Let $\rho_{1}= \{P_\pi: \pi(1)= 1 \}$, where $\pi(i)$ denotes the index of the element in the $i$th position according to permutation $\pi$. Fix any $P \in \rho_1$. Then, for any $X$, $\inner{X,W_1} = \inner{P X,W_1}$. This implies that, for all $X$,
$\tr({W_1}^{\top}X)= \tr({W_1}^{\top}PX)$. Using the fact that $\tr(A^\top X)=\tr(B^\top X), \forall X$ implies $A=B$, we have that ${W_1}^{\top}= {W_1}^{\top}P$. Because $P^\top = P^{-1}$, this means $PW_1 = W_1$. This shows that all rows of $W_1$, other than 1st row, are the same but perhaps different from 1st row. By considering $\rho_{i} = \{ P_\pi:\pi(i) = i \}$ for $i > 1$, the same reasoning shows that, for each $i$, all rows of $W_i$, other than $i$th row, are the same but possibly different from $i$th row.

Let $\rho_{1\leftrightarrow 2}= \{P_\pi: \pi(1)= 2, \pi(2)=1 \}$. Fix any $P \in \rho_{1\leftrightarrow 2}$. Then, for any $X$, $\inner{X,W_2} = \inner{PX,W_1}$ and $\inner{X,W_1} = \inner{PX,W_2}$. Thus, we have $W_2^{\top} = W_1^{\top}P$
as well as $W_1^\top = W_2^\top P$ which means $PW_2 = W_1, PW_1 = W_2$. This shows that row 1 of $W_1$ and row 2 of $W_2$ are the same. Moreover, row 2 of $W_1$ and row 1 of $W_2$ are the same.
Thus, for some $u,u' \in \reals^d$, $W_1$ is of the form $[ u | u' | u'  | \ldots | u' ]^\top$ and $W_2$ is of the form $[u' | u | u' | \ldots | u']^\top$. Repeating this argument by considering $\rho_{1 \leftrightarrow i}$ for $i > 2$ shows that
$W_i$ is of the same form ($u$ in row $i$ and $u'$ elsewhere).

Therefore, we have proved that any linear map that is permutation invariant has to be of the form:
\[
X \mapsto \left( u^\top X_i + (u')^\top \sum_{j\neq i} X_j \right)_{i=1}^m .
\]
We can reparameterize above using $w = u-u'$ and $v = u'$ which proves the result.
\end{proof}

\section{Proof of Lemma~\ref{lem:normexpr}}
\begin{proof}
The first equality is true because
\begin{align*}
\| X^\top \|_{1 \to p} &= \sup_{v \neq 0} \frac{ \| X^\top v \|_p }{\| v \|_1} 
= \sup_{v \neq 0} \sup_{u \neq 0} \frac{ \inner{X^\top v, u} }{\| v \|_1 \|u\|_q} \\
&= \sup_{u \neq 0} \sup_{v \neq 0} \frac{ \inner{v, X u} }{\| v \|_1 \|u\|_q} 
= \sup_{u \neq 0} \frac{ \| X u\|_\infty }{\| u \|_q} 
= \| X \|_{q \to \infty} .
\end{align*}
The second is true because
\begin{align*}
\| X \|_{q \to \infty} &= \sup_{u \neq 0} \frac{ \| X u \|_\infty }{\|u\|_q }
= \sup_{u \neq 0} \max_{j=1}^m \frac{ |\inner{X_j,u}| }{\|u\|_q} \\
&= \max_{j=1}^m \sup_{u \neq 0} \frac{ |\inner{X_j,u}| }{\|u\|_q}
= \max_{j=1}^m \| X_j \|_p . 
\end{align*} 
\end{proof}

\section{Proof of Theorem~\ref{thm:sco}}
Our theorem is developed from the ``expectation version" of Theorem 6 of \citet{shalev2009} that was originally given in probabilistic form. The expected version is as follows.

Let $\Z$ be a space endowed with a probability distribution generating iid draws $Z_1,\ldots,Z_n$. Let $\W \subseteq \reals^d$ and  $f:\W \times \Z \to \reals$ be $\lambda$-strongly convex\footnote{Recall that
a function is called $\lambda$-strongly convex (w.r.t. $\|\cdot\|_2$) iff $f-\tfrac{\lambda}{2}\|\cdot\|_2^2$ is convex.}
and $G$-Lipschitz (w.r.t. $\|\cdot\|_2$) in $w$ for every $z$. We define $F(w)= \E{f(w,Z)}$ and let
\begin{align*}
w^\star &= \argmin_{w \in \W}\ F(w) ,\\
\hat{w} &= \argmin_{w \in \W}\ \frac{1} {n}\sum_{i=1}^nf(w,Z_i) .
\end{align*}
Then $\E{F(\hat{w}) - \ F(w^\star)} \le \frac{4G^2}{\lambda n}$, where the expectation is taken over the sample. The above inequality can be proved by carefully going through the proof of Theorem 6 proved by~\citet{shalev2009}. 

We now derive the ``expectation version" of Theorem 7 of \citet{shalev2009}. Define the regularized empirical risk minimizer as follows:
\begin{equation}
\label{optimumw}
\hat{w}_{\lambda}= \argmin_{w \in \W}\ \frac{\lambda}{2}\|w\|^2_2 + \frac{1} {n}\sum_{i=1}^nf(w,Z_i) .
\end{equation}
The following result gives optimality guarantees for the regularized empirical risk minimizer.
\begin{theorem}
\label{expectedtheorem}
Let $\W = \{w\::\: \|w\|_2 \le W_2 \}$ and let $f(w, z)$ be convex and $G$-Lipschitz (w.r.t. $\|\cdot\|_2$) in $w$ for every $z$. Let $Z_1,...,Z_n$ be iid samples and let $\lambda= \sqrt{\frac{\frac{4G^2}{n}}{\frac{W_2^2}{2}+\frac{4W_2^2}{n}}}$. Then for $\hat{w}_{\lambda}$ and $w^\star$ as defined above, we have 
\begin{equation}
\E{ F(\hat{w}_{\lambda}) -  F(w^\star)} \le 2\,G\,W_2\left(\frac{8}{n} + \sqrt{\frac{2}{n}}\right) .
\end{equation}
\end{theorem}
\begin{proof}
Let $r_{\lambda}(w,z)= \frac{\lambda}{2}\|w\|^2_2 + f(w,z)$. Then $r_{\lambda}$ is $\lambda$-strongly convex with Lipschitz constant $\lambda W_2 + G$ in $\|\cdot\|_2$. Applying ``expectation version" of Theorem 6 of~\citet{shalev2009} to $r_{\lambda}$, we get
\[
\E{ \frac{\lambda}{2}\|\hat{w}_{\lambda}\|^2_2 + F(\hat{w}_{\lambda}) }
\le \min_{w \in \W}\ \left\{ \frac{\lambda}{2}\|w\|^2_2 +  F(w) \right\} + \frac{4(\lambda W_2 + G)^2}{\lambda n} 
\le \frac{\lambda}{2}\|w^\star\|^2_2 +  F(w^*) + \frac{4(\lambda W_2 + G)^2}{\lambda n} .
\]
Thus, we get
\[
\E{ F(\hat{w}_{\lambda}) - F(w^\star) } \le \frac{\lambda W_2^2}{2} +\frac{4(\lambda W_2 + G)^2}{\lambda n} \ .
\]
Minimizing the upper bound w.r.t. $\lambda$, we get $\lambda= \sqrt{\frac{4G^2}{n}}\sqrt{\frac{1}{\frac{W_2^2}{2} + \frac{4W_2^2}{n}}}$. Plugging this choice back in the equation above and using the fact that $\sqrt{a+b} \le \sqrt{a} + \sqrt{b}$ finishes
the proof of Theorem \ref{expectedtheorem}.
\end{proof}

We now have all ingredients to prove Theorem~\ref{thm:sco}.
\begin{proof}[Proof of Theorem~\ref{thm:sco}]
Let $\Z = \X \times \Y$ and $f(w,z)=\phi(Xw,y)$ and apply Theorem~\ref{expectedtheorem}. Finally note that if $\phi$ is $G_\phi$-Lipschitz w.r.t. $\|\cdot\|_\infty$ and every row of $X \in \reals^{m \times d}$ has Euclidean norm bounded by $R_X$
then $f(\cdot,z)$ is $G_\phi R_X$-Lipschitz w.r.t. $\|\cdot\|_2$ in $w$.
\end{proof}

\section{Proof of Theorem~\ref{thm:lip1}}
\begin{proof}
Following exactly the same line of reasoning (reducing a sample of size $n$, where each prediction is $\reals^m$-valued,
to an sample of size $mn$, where each prediction is real valued) as in the beginning of proof of Proposition~\ref{prop:Ninftycover}, we have
\begin{equation}
\label{eq:coveringofFfromG1}
\N_\infty(\epsilon,\phi \circ \F_1,n) \le \N_\infty(\epsilon/G_\phi,\G_1,mn) .
\end{equation}
Plugging in the following bound due to \citet[Corollary 5]{zhang2002covering}:
\begin{align*}
\log_2 \N_\infty(\epsilon/G_\phi,\G_1,mn) &\le \left\lceil \frac{288\,G_\phi^2 \, W_1^2 \, \bar{R}_X^2\,(2+\ln d)}{\epsilon^2} \right\rceil  \\
&\times \log_2 \left( 2 \lceil 8 G_\phi W_1 \bar{R}_X/\epsilon \rceil mn+1 \right)
\end{align*}
into~\eqref{eq:coveringofFfromG1} respectively proves the result.
\end{proof}

\section{Calculations involved in deriving Equation~\eqref{eq:onlinefastrate}}
\label{app:etacalc}

Plugging in the value of $\eta$ from~\eqref{eq:etaval} into the expression
\[
\frac{L_\phi(w^\star)}{(1-4\eta H)} + \frac{W_2^2}{2\eta(1-4\eta H)n}
\]
yields (using the shorthand $L^\star$ for $L_\phi(w^\star)$)
\[
L^\star + \frac{2HW_2 L^\star}{\sqrt{4H^2W_2^2 + 2H L^\star n}} + \frac{W_2}{n} \left[
\frac{4H^2 W_2^2}{\sqrt{4H^2W_2^2 + 2H L^\star n}} + 
\sqrt{4H^2W_2^2 + 2H L^\star n} +
4 H W_2
\right]
\]
Denoting $HW_2^2/n$ by $x$, this simplifies to
\[
L^\star + \frac{2\sqrt{x} L^\star + 4x\sqrt{x}}{\sqrt{4x + 2L^\star}} + \sqrt{x} \sqrt{4x + 2L^\star} + 4x .
\]
Using the arithmetic mean-geometric mean inequality to upper bound the middle two terms gives
\[
L^\star + 2\sqrt{ 2 x L^\star + 4x^2} + 4x .
\]
Finally, using $\sqrt{a+b} \le \sqrt{a} + \sqrt{b}$, we get our final upper bound
\[
L^\star + 2\sqrt{ 2 x L^\star } + 8x .
\]

\section{Calculation of smoothness constant}
\label{app:smoothness}

\begin{align*}
&\quad \| (X^{(i)})^\top \nabla_s^2 \phi(X^{(i)}w, y^{(i)}) X^{(i)} \|_{2 \to 2}= \underset{v \neq 0}{\sup}\frac{\quad \| (X^{(i)})^\top \nabla_s^2 \phi(X^{(i)}w, y^{(i)}) X^{(i)}v \|_{2}}{\|v\|_2} \\
& \le\underset{v \neq 0}{\sup} \frac{\|(X^{(i)})^\top\|_{1 \to 2}\|\nabla_s^2 \phi(X^{(i)}w, y^{(i)})X^{(i)}v\|_1}{\|v\|_2} \le \underset{v \neq 0}{\sup}\frac{\| (X^{(i)})^\top  \|_{1 \to 2} \cdot
\|  \nabla_s^2 \phi(X^{(i)}w, y^{(i)}) \|_{\infty \to 1} \cdot
\|   X^{(i)}v \|_{\infty}}{\|v\|_2} \\
& \le \underset{v \neq 0}{\sup}\frac{\| (X^{(i)})^\top  \|_{1 \to 2} \cdot
\|  \nabla_s^2 \phi(X^{(i)}w, y^{(i)}) \|_{\infty \to 1} \cdot
\|   X^{(i)} \|_{2 \to \infty} \cdot \|v\|_2}{\|v\|_2} \\
& \le \left( \max_{j=1}^m \| X^{(i)}_j \| \right)^2 \cdot \|  \nabla_s^2 \phi(X^{(i)}w, y^{(i)}) \|_{\infty \to 1} \\
& \le R_X^2 \| \nabla_s^2 \phi(X^{(i)}w, y^{(i)}) \|_{\infty \to 1} .
\end{align*}

\section{Proof of Lemma~\ref{lem:vectorsmooth}}
\label{app:vectorsmooth}

\begin{proof}
Consider the function
\[
f(t) = \phi((1-t)s_1+ts_2) .
\]
It is clearly non-negative. Moreover
\begin{align*}
|f'(t_1) -f'(t_2)| &= | \inner{ \nabla_s \phi(s_1+t_1(s_2-s_1)) - \nabla_s \phi(s_1+t_2(s_2-s_1)), s_2 -s_1} | \\
&\le \lipl \nabla_s \phi(s_1+t_1(s_2-s_1)) - \nabla_s \phi(s_1+t_2(s_2-s_1)) \lipr_\star \cdot \lipl s_2 - s_1 \lipr \\
&\le H_\phi \, |t_1 - t_2|\, \lipl s_2 - s_1 \lipr^2
\end{align*}
and therefore it is smooth with constant $h = H_\phi  \lipl s_2 - s_1 \lipr^2$. Appealing to Lemma~\ref{lem:scalarsmooth}
now gives
\[
(f(1) - f(0))^2 \le 6 H_\phi  \lipl s_2 - s_1 \lipr^2 (f(1) + f(0)) (1-0)^2 
\]
which proves the lemma since $f(0) = \phi(s_1)$ and $f(1) = \phi(s_2)$.
\end{proof}

\section{Proof of Proposition~\ref{prop:N2cover}}
\begin{proof}
Let $w,w' \in \F_{\phi,2}(r)$. Using Lemma~\ref{lem:vectorsmooth}
\begin{align*}
&\quad \sum_{i=1}^n \frac{1}{n} \left( \phi(X^{(i)}w,y^{(i)}) - \phi(X^{(i)}w',y^{(i)}) \right)^2 \\
&\le 6 H_\phi \sum_{i=1}^n \frac{1}{n} \left( \phi(X^{(i)}w,y^{(i)}) + \phi(X^{(i)}w',y^{(i)}) \right) \\
&\quad \cdot \| X^{(i)}w - X^{(i)}w' \|_\infty^2 \\
&\le 6 H_\phi \cdot \max_{i=1}^n \| X^{(i)}w - X^{(i)}w' \|_\infty^2  \\
&\quad \cdot \sum_{i=1}^n \frac{1}{n} \left( \phi(X^{(i)}w,y^{(i)}) + \phi(X^{(i)}w',y^{(i)}) \right) \\
&= 6 H_\phi \cdot \max_{i=1}^n \| X^{(i)}w - X^{(i)}w' \|_\infty^2 \cdot \left( \hat{L}_\phi(w) + \hat{L}_\phi(w') \right) \\
&\le 12 H_\phi r \cdot \max_{i=1}^n \| X^{(i)}w - X^{(i)}w' \|_\infty^2 .
\end{align*} 
where the last inequality follows because $\hat{L}_\phi(w) + \hat{L}_\phi(w') \le 2r$.

This immediately implies that if we have a cover of the class $\G_2$ at scale $\epsilon/\sqrt{12 H_\phi r}$
w.r.t. the metric
\[
\max_{i=1}^n \max_{j=1}^m \left| \inner{X^{(i)}_j,w} - \inner{ X^{(i)}_j, w'} \right|
\]
then it is also a cover of $\F_{\phi,2}(r)$ w.r.t. $d_{2}^{Z^{(1:n)}}$. Therefore, we have
\begin{equation}
\label{eq:coveringofFphifromG2}
\N_2(\epsilon,\F_{\phi,2}(r),Z^{(1:n)}) \le \N_\infty(\epsilon/\sqrt{12 H_\phi r} ,\G_2,mn) .
\end{equation}
Appealing once again to a result by \citet[Corollary 3]{zhang2002covering}, we get
\begin{align*}
\log_2 \N_\infty(\epsilon/\sqrt{12 H_\phi r},\G_2,mn) &\le \left\lceil \frac{12H_\phi \, W_2^2 \, R_X^2 \,r}{\epsilon^2} \right\rceil \\
&\quad \times \log_2 (2mn+1)
\end{align*}
which finishes the proof.
% Leaving the L-1/L-\infty case for journal version....
%
% Repeating the same argument for $\F_{\phi,1}(r)$ and using Zhang's Corollary 5:
%\[
%\log_2 \N_\infty(\epsilon/\sqrt{12 H_\phi r},\G_1,mn) \le \left\lceil \frac{3456\,H_\phi \, W_1^2 \, R_X^2\,(2+\ln d)\,r}{\epsilon^2} \right\rceil 
%\log_2 \left( 2 \lceil 8 \sqrt{12 H_\phi r} W_1 R_X/\epsilon \rceil mn+1 \right) .
%\]
%proves the second inequality.
\end{proof}

\section{Proof of Corollary~\ref{cor:subroot}}
\label{app:subroot}

\begin{proof}
We plug in Proposition~\ref{prop:N2cover}'s estimate into~\eqref{eq:dudley1}:
\begin{align*}
\radhat{n}{\F_{\phi,2}(r)}
&\le \inf_{\alpha > 0} \left( 4 \alpha + 10 \int_{\alpha}^{\sqrt{Br}} \sqrt{\frac{\left\lceil \frac{12H_\phi \, W_2^2 \, R_X^2 \,r}{\epsilon^2} \right\rceil \log_2 (2mn+1)}{n}} d\epsilon \right) \\
&\le \inf_{\alpha > 0} \left( 4 \alpha + 20\sqrt{3}W_2R_X\sqrt{\frac{r H_\phi \log_2(3mn)}{n}} \int_{\alpha}^{\sqrt{Br}} \frac{1}{\epsilon} d\epsilon \right)  \ .
\end{align*} 
Now choosing $\alpha = C\sqrt{r}$ where
$
C = 5\sqrt{3}W_2R_X\sqrt{\frac{H_\phi \log_2(3mn)}{n}}
$
gives us the upper bound
\[
\radhat{n}{\F_{\phi,2}(r)} \le 4\sqrt{r} C \left( 1 + \log \frac{\sqrt{B}}{C} \right) \le 4\sqrt{r} C \log \frac{3\sqrt{B}}{C} .
\]
\end{proof}

\section{Proof of Theorem~\ref{thm:smooth}}
\label{app:smooth}

\begin{proof}
We appeal to Theorem 6.1 of \citet{bousquet2002concentration} that assumes there exists an upper bound
\[
\radhat{n}{\F_{2,\phi}(r)} \le \psi_n(r)
\]
where $\psi_n: [0,\infty) \to \reals_+$ is a non-negative, non-decreasing, non-zero function such that $\psi_n(r)/\sqrt{r}$ is non-increasing.
The upper bound in  Corollary~\ref{cor:subroot}
above satisfies these conditions and therefore we set $\psi_n(r) = 4\sqrt{r} C \log \frac{3\sqrt{B}}{C}$ with $C$ as 
defined in Corollary~\ref{cor:subroot}. From Bousquet's result, we know that, with probability at least $1-\delta$,
\begin{align*}
\forall w \in \F_2,\ L_\phi(w)
&\le \hat{L}_\phi(w) + 
45 r_n^\star +
\sqrt{8 r_n^\star L_\phi(w)} \\
&\quad + \sqrt{4 r_0 L_\phi(w)} +
20 r_0
\end{align*}
where $r_0 = B(\log(1/\delta) + \log \log n)/n$ and $r_n^\star$ is the largest solution to the equation $r = \psi_n(r)$. In our case,
$r_n^\star = \left(   4 C \log \frac{3\sqrt{B}}{C}  \right)^2$. This proves the first inequality.

Now, using the above inequality with $w = \hat{w}$, the empirical risk minimizer and noting that $\hat{L}_\phi(\hat{w})
\le \hat{L}_\phi(w^\star)$, we get
\begin{align*}
L_\phi(\hat{w})
&\le \hat{L}_\phi(w^\star) + 
45 r_n^\star +
\sqrt{8 r_n^\star L_\phi(\hat{w})} \\
&\quad +\sqrt{4 r_0 L_\phi(\hat{w})} +
20 r_0
\end{align*}
The second inequality now follows after some elementary calculations detailed below.
\end{proof}

\subsection{Details of some calculations in the proof of Theorem~\ref{thm:smooth}}
Using Bernstein's inequality, we have, with probability at least $1-\delta$,
\begin{align*}
\hat{L}_\phi(w^\star) &\le L_\phi(w^\star) + \sqrt{\frac{4 \mathrm{Var}[\phi(Xw^\star,y)] \log(1/\delta)}{n}} + \frac{4B \log (1/\delta)}{n} \\
&\le L_\phi(w^\star) + \sqrt{\frac{4 B L_\phi(w^\star) \log(1/\delta)}{n}} + \frac{4B \log (1/\delta)}{n} \\
&\le L_\phi(w^\star) + \sqrt{4 r_0 L_\phi(w^\star)} + 4 r_0 .
\end{align*}

Set $D_0 = 45 r_n^\star + 20 r_0$.
Putting the two bounds together and using some simple upper bounds, we have, with probability at least $1- 2\delta$,
\begin{align*}
L_{\phi}(\hat{w}) &\le \sqrt{D_0 \hat{L}_\phi(w^\star) } + D_0 ,\\
\hat{L}_\phi(w^\star) &\le \sqrt{D_0 L_\phi(w^\star) } + D_0 .
\end{align*}
which implies that 
\[
L_{\phi}(\hat{w}) \le \sqrt{D_0} \sqrt{\sqrt{D_0 L_\phi(w^\star) } + D_0  } + D_0 .
\]
Using $\sqrt{ab} \le (a+b)/2$ to simplify the first term on the right gives us
\begin{align*}
L_{\phi}(\hat{w}) &\le \frac{D_0}{2} + \frac{\sqrt{D_0 L_\phi(w^\star)} + D_0}{2} + D_0 = \frac{\sqrt{D_0 L_\phi(w^\star)}}{2} + 2D_0 \ .
\end{align*}

\end{document}

%% file: macros.tex
\newcommand\reals{\mathbb{R}} % real nos
\newcommand\X{\mathcal{X}} % input space
\newcommand\Y{\mathcal{Y}} % label space
\newcommand\Z{\mathcal{Z}} % input, label space
\newcommand\E[1]{\mathbb{E}\left[#1\right]} % E[ ]
\newcommand\Ehat[1]{\widehat{\mathbb{E}}\left[#1\right]} % \widehat{E}[ ]
\newcommand\Es[2]{\mathbb{E}_{#1}\left[#2\right]} % E[ ] with subscript
\newcommand\inner[1]{\left\langle {#1} \right\rangle} % inner product
\newcommand\F{\mathcal{F}} % function class
\newcommand\W{\mathcal{W}} % class of weight vectors
\newcommand\G{\mathcal{G}} % function class
\newcommand\N{\mathcal{N}} % covering numbers
\newcommand\tr{\mathrm{Tr}}
\newcommand\one{\mathbf{1}} % vector 1's

\DeclareMathOperator*{\argmin}{argmin}

\newcommand\Ffull{\F_{\mathrm{full}}}
\newcommand\Flin{\F_{\mathrm{lin}}}
\newcommand\Fperm{\F_{\mathrm{perminv}}}

\newcommand\radhat[2]{\widehat{\mathfrak{R}}_{#1}\left( #2 \right)}

\newcommand\lipl{|||}
\newcommand\lipr{|||}
\newcommand\liplop{\lipl}
\newcommand\liprop{\lipr_{\mathrm{op}}}

\newcommand\listnet{\phi_{\mathrm{LN}}}
\newcommand\sdcg{\phi_{\mathrm{SD}}}

\newtheorem{lemma}{Lemma}
\newtheorem{theorem}[lemma]{Theorem}
\newtheorem{proposition}[lemma]{Proposition}
\newtheorem{corollary}[lemma]{Corollary}

%% file: abstract.tex
\begin{abstract}
We consider the generalization ability of algorithms for learning to rank at a query level, a problem also called subset ranking.
Existing generalization error bounds
necessarily degrade as the size of the document list associated
with a query increases. We show that such a degradation is not intrinsic to the problem.
For several loss functions, including the cross-entropy loss used in the well known
ListNet method, there is \emph{no} degradation in generalization ability
as document lists become longer. We also provide novel generalization error bounds under $\ell_1$ regularization
and faster convergence rates if the loss function is smooth.
\end{abstract}

%% file: intro.tex
Learning to rank at the query level has emerged as an exciting research area at the intersection of information
retrieval and machine learning. Training data in learning to rank consists of queries along with associated documents, where documents are represented as feature vectors.
For each query, the documents are labeled with human relevance judgements. The goal at training time is to learn
a ranking function that can, for a future query, rank its associated documents in order of their relevance to the query.
The performance of ranking functions on test sets is evaluated using a variety of performance measures such as
NDCG~\citep{jarvelin2002}, ERR~\citep{chapelle2009expected} or Average Precision~\citep{yue2007}.

The performance measures used for testing ranking methods cannot be directly optimized during training time as they
lead to discontinuous optimization problems. As a result, researchers often minimize {\em surrogate} loss functions
that are easier to optimize. For example, one might consider smoothed versions of, or convex upper bounds on, the target
performance measure. However, as soon as one optimizes a surrogate loss, one has to deal with two questions~\citep{chapelle2011}. First, does minimizing
the surrogate on finite training data imply small expected surrogate loss on infinite unseen data? Second, does small expected
surrogate loss on infinite unseen data imply small {\em target} loss on infinite unseen data? The first issue is one of {\em
generalization error bounds} for empirical risk minimization (ERM) algorithms that minimize surrogate loss on training data.
The second issue is one of {\em calibration}: does consistency in the surrogate loss imply consistency in the target loss?

This paper deals with the former issue, viz. that of generalization error bounds for surrogate loss minimization.
%Learning to rank has witnessed tremendous activity in the past decade. These activities include workshops at premier machine learning (NIPS)
%and information retrieval (SIGIR) conferences, books~\citep{liu2011learning,li2011learning}, release of public benchmark datasets~\citep{liu2007letor},
%and a Learning to Rank challenge~\citep{chapelle2011yahoo}. Despite all this activity, it is surprising that very few papers have dealt with generalization error bounds
%for query-level surrogate loss minimization.
In pioneering works, \citet{lan2008query, lan2009} gave generalization error bounds for learning to rank algorithms.
However, while the former paper was restricted to analysis of pairwise approach to learning to rank, the later paper was limited to results on just three surrogates: ListMLE, ListNet and RankCosine. To the best of our knowledge,
the most generally applicable bound on the generalization error of query-level learning to rank algorithms has been obtained by \citet{chapelle2010}.

\begin{table*}[t]
\caption{A comparison of three bounds given in this paper for Lipschitz loss functions. Criteria for comparison: algorithm bound applies to (OGD = Online Gradient Descent, [R]ERM = [Regularized] Empirical Risk Minimization),
whether it applies to general (possibly non-convex) losses, and whether the constants involved are tight.
}
\label{tab:comparison}
\begin{center}
\begin{tabular}{cccc}
\hline
Bound & Applies to & Handles Nonconvex Loss & ``Constant" hidden in $O(\cdot)$ notation \\
\hline
Theorem~\ref{thm:online} & OGD & No & Smallest \\
Theorem~\ref{thm:sco} & RERM & No & Small \\
Theorem~\ref{thm:lip} & ERM & Yes & Hides several logarithmic factors \\
\hline
\end{tabular}
\end{center}
\end{table*}

The bound of \citet{chapelle2010}, while generally applicable, does have an explicit dependence on the {\em length} of the document list associated
with a query. Our investigations begin with this simple question: is an explicit dependence on the length of document lists unavoidable in generalization
error bounds for query-level learning to rank algorithms? We focus on the prevalent technique in literature where learning to rank algorithms learn linear scoring functions and obtain ranking by sorting scores in descending order. Our first contribution (Theorem~\ref{thm:dimension}) is to show that dimension of linear scoring functions that are \emph{permutation invariant} (a necessary condition for being valid scoring functions for learning to rank) has no dependence on the length of document lists. Our second contribution (Theorems~\ref{thm:online}, \ref{thm:sco}, \ref{thm:lip}) is to show that as long as one uses the ``right" norm in defining the Lipschitz
constant of the surrogate loss, we can derive generalization error bounds that have \emph{no explicit dependence on the length of document lists}.
The reason that the second contribution involves three bounds is that they all have different strengths and scopes of application (See Table~\ref{tab:comparison} for  a comparison).
Our final contribution is to provide novel generalization error bounds for learning to rank in two previously unexplored settings: almost dimension
independent bounds when using high dimensional features  with $\ell_1$ regularization (Theorem~\ref{thm:lip1}) and ``optimistic" rates (that can be as fast as $O(1/n)$) when the loss function
is smooth (Theorem~\ref{thm:smooth}). We also apply our results on popular convex and non-convex surrogates. All omitted proofs can be found in the appendix (see supplementary material).

%% file: prelim.tex
In learning to rank (also called subset ranking to distinguish it from other related
problems, e.g., bipartite ranking), a training example is of the form $((q, d_1,\ldots,d_m),y)$. Here
$q$ is a search query and $d_1,\ldots,d_m$ are $m$ documents with varying degrees of \emph{relevance} to the query.
Human labelers provide the relevance vector $y \in \reals^m$ where the entries in $y$ contain the relevance labels for the
$m$ individual documents. Typically, $y$ has integer-valued entries in the range $\{0,\ldots,Y_{\max}\}$ where $Y_{\max}$ is often less than
$5$. For our theoretical analysis, we get rid of some of these details by assuming that some feature map $\Psi$ exists to map
a query document pair $(q,d)$ to $\reals^d$. As a result, the training example $((q,d_1,\ldots,d_m),y)$ gets converted into $(X,y)$
where $X  = [ \Psi(q,d_1),\ldots,\Psi(q,d_m) ]^\top$ is an $m \times d$ matrix with the $m$ query-document feature vector as rows.
With this abstraction, we have an input space $\X \subseteq \reals^{m \times d}$ and a label space $\Y \subseteq \reals^m$.

A training set consists of iid examples $(X^{(1)},y^{(1)}),\ldots,(X^{(n)},y^{(n)})$ drawn from some underlying distribution $D$. 
To rank the documents in an instance $X \in \X$, often a score vector $s \in \reals^m$ is computed. A ranking of the documents
can then be obtained from $s$ by sorting its entries in decreasing order. A common choice for the scoring function is to make it
\emph{linear} in the input $X$ and consider the following class of vector-valued functions:
\begin{align}\label{eq:lindef}
\Flin &= \{ X \mapsto Xw \::\: X \in \reals^{m \times d}, w \in \reals^d \} .
\end{align}
Depending upon the regularization, we also consider the following two subclasses of $\Flin$ :
\begin{align*}
\F_2 &:= \{ X \mapsto Xw \::\: X \in \reals^{m \times d}, w \in \reals^d, \|w\|_2 \le W_2 \} ,\\
\F_1 &:= \{ X \mapsto Xw \::\: X \in \reals^{m \times d}, w \in \reals^d, \|w\|_1 \le W_1 \} .
\end{align*}
In the input space $\X$, it is natural for the rows of $X$ to have a bound on the appropriate dual norm. Accordingly, whenever
we use $\F_2$, the input space is set to 
$
\X = \{ X \in \reals^{m \times d} \::\: \forall j \in [m],\ \| X_j \|_2 \le R_X  \}
$
where $X_j$ denotes $j$th row of $X$ and $[m] := \{1,\ldots,m\}$.
Similarly, when we use $\F_1$, we set
$
\X = \{ X \in \reals^{m \times d} \::\: \forall j \in [m],\ \| X_j \|_\infty \le \bar{R}_X  \} .
$
These are natural counterparts to the following function classes studied in binary classification and regression:
{\small
\begin{align*}
\G_2 &:= \{ x \mapsto \inner{x,w} \::\:  \|x\|_2 \le R_X, w \in \reals^d, \|w\|_2 \le W_2 \} ,\\
\G_1 &:= \{ x \mapsto \inner{x,w} \::\:  \|x\|_\infty \le \bar{R}_X, w \in \reals^d, \|w\|_1 \le W_1 \} .
\end{align*}
}
A key ingredient in the basic setup of the learning to rank problem is a loss function $\phi : \reals^m \times \Y \to \reals_+$ where
$\reals_+$ denotes the set of non-negative real numbers. Given a class $\F$ of vector-valued functions, a loss $\phi$ yields a
natural loss class: namely the class of real-valued functions that one gets by composing $\phi$ with functions in $\F$:
\[
\phi \circ \F := \{ (X,y) \mapsto \phi(f(X),y) \::\: X \in \reals^{m \times d}, f \in \F \}.
\]
For vector valued scores, the Lipschitz constant of $\phi$ depends on the norm
$\lipl \cdot \lipr$ that we
decide to use in the score space ($\lipl \cdot \lipr_\star$ is dual of $\lipl \cdot \lipr$):
\[
\forall y \in \Y,s,s' \in \reals^m,\ | \phi(s_1,y) - \phi(s_2,y) | \le G_\phi \lipl s_1 - s_2  \lipr .
\]
If $\phi$ is differentiable, this is equivalent to:
$
\forall y \in \Y, s\in\reals^m,\ \lipl \nabla_s \phi(s,y) \lipr_\star \le G_\phi  .
$
Similarly, the smoothness constant $H_\phi$ of $\phi$ defined as:
$\forall y \in \Y,s,s' \in \reals^m,$
\[
\lipl \nabla_s \phi(s_1,y) - \nabla_s \phi(s_2,y) \lipr_\star \le H_\phi \lipl s_1 - s_2 \lipr .
\]
also depends on the norm used in the score space.
If $\phi$ is twice differentiable, the above inequality is equivalent to 
\[
\forall y \in \Y,s \in \reals^m,\ \liplop \nabla_s^2 \phi(s,y) \liprop \le H_\phi
\]
where $\liplop \cdot \liprop$ is the operator norm induced by the pair $\lipl \cdot \lipr,\lipl \cdot \lipr_\star$ and defined as
$
\liplop M \liprop := \sup_{v \neq 0} \frac{\lipl M v \lipr_\star}{\lipl v \lipr} .
$
Define the expected loss of $w$ under the distribution $D$
$L_\phi(w) := \Es{(X,y)\sim D}{\phi(Xw,y)}$
and its empirical loss on the sample as
$\hat{L}_\phi(w) := \frac{1}{n} \sum_{i=1}^n \phi(X^{(i)} w,y^{(i)})$.
The minimizer of $L_\phi(w)$ (resp. $\hat{L}_\phi(w)$) over some function class (parameterized by $w$) will be denoted by $w^\star$ (resp. $\hat{w}$).
We may refer to expectations w.r.t. the sample using $\Ehat{\cdot}$. To reduce notational clutter, we often refer to $(X,y)$ jointly by $Z$
and $\X \times \Y$ by $\Z$. For vectors, $\inner{u,v}$ denotes the standard inner product $\sum_i u_iv_i$ and for matrices $U,V$ of the same shape,
$\inner{U,V}$ means $\tr(U^\top V) = \sum_{ij} U_{ij}V_{ij}$. The set of $m!$ permutation $\pi$ of degree $m$ is denoted by $S_m$. A vector of ones
is denoted by $\one$.

%% file: applications.tex
%\section{Application to Specific Losses}
To whet the reader's appetite for the technical presentation that follows, we will consider two loss functions, one convex and one non-convex, to illustrate the concrete improvements
offered by our new generalization bounds. A generalization bound is of the form:
$L_\phi(\hat{w}) \le L_\phi(w^\star) +$``complexity term". 
It should be noted that $w^\star$ is not available to the learning algorithm as it 
needs knowledge of underlying distribution of the data. The complexity term of \citet{chapelle2010} is $O(G_\phi^{CW} W_2 R_X \sqrt{m/n})$. The constant $G_\phi^{CW}$ is the Lipschitz constant of the surrogate $\phi$ (viewed as a function of the score vector $s$) w.r.t. $\ell_2$ norm. Our bounds will instead be of the form
$O(G_\phi W_2 R_X \sqrt{1/n})$, where $G_\phi$ is the Lipschitz constant of $\phi$ w.r.t. $\ell_\infty$ norm. Note that our bounds
are free of any explicit $m$ dependence. Also, by definition, $G_\phi \le G_\phi^{CW}\sqrt{m}$ but the
former can be much smaller as the two examples below illustrate. In benchmark datasets~\citep{liu2007letor}, $m$ can easily be in the $100$-$1000$ range.

\subsection{Application to ListNet}
\label{sec:listnet}

The ListNet ranking method~\citep{Cao2007} uses a \emph{convex} surrogate, that is defined in the following way\footnote{The ListNet paper actually defines a family of losses based on probability models for top $k$ documents. We use $k=1$ in our definition since that is the version implemented in their experimental results.}. Define $m$ maps from $\reals^m$ to $\reals$ as: $P_j(v) = \exp(v_j)/\sum_{i=1}^m \exp(v_i)$ for $j \in [m]$. Then, we
have, for $s \in \mathbb{R}^m$ and $y \in \mathbb{R}^m$,
\[
\listnet(s,y) = - \sum_{j=1}^m P_j(y) \log P_j(s) .
\]
An easy calculation shows that the Lipschitz (as well as smoothness) constant of $\listnet$ is $m$ independent.
\begin{proposition}
\label{prop:listnet}
The Lipschitz (resp. smoothness) constant of $\listnet$ w.r.t. $\|\cdot\|_\infty$ satisfies $G_{\listnet} \le 2$ (resp. $H_{\listnet} \le 2$) for any $m \ge 1$.
\end{proposition}
%\begin{proof}
%See Appendix~\ref{app:listnet}.
%\end{proof}
Since the bounds above are independent of $m$, so the generalization bounds resulting from their use in Theorem~\ref{thm:lip} and Theorem~\ref{thm:smooth}
will also be independent of $m$ (up to logarithmic factors).
We are not aware of prior generalization bounds for ListNet that do not scale with $m$. In particular, the results
of \citet{lan2009}
have an $m!$ dependence since they consider the top-$m$ version of ListNet. However, even if the top-$1$ variant above is considered, their proof technique will result in at least a
linear dependence on $m$ and does not result in as tight a bound as we get from our general 
results. It is also easy to see that the Lipschitz constant  $G_{\listnet}^{CW}$  of ListNet loss w.r.t. $\ell_2$ norm is also $2$ and hence the bound of~\citet{chapelle2010} necessarily has a $\sqrt{m}$ dependence in it.
Moreover, generalization error bounds for ListNet exploiting its smoothness will interpolate between the 
pessimistic $1/\sqrt{n}$ and optimistic $1/n$ rates. These have never been provided before. 

\subsection{Application to Smoothed DCG@1}
\label{sec:sdcg1}

This example is from the work of~\citet{chapelle2010}. Smoothed DCG@1, a \emph{non-convex} surrogate, is defined as:
\[
\sdcg(s,y) = D(1) \sum_{i=1}^m G(y_i) \frac{\exp(s_i/\sigma)}{\sum_j \exp(s_j/\sigma)} ,
\]
where $D(i) = 1/\log_2(1+i)$ is the ``discount" function and $G(i) = 2^i-1$ is the ``gain" function. The amount of smoothing is controlled by the parameter $\sigma > 0$
and the smoothed version approaches DCG@1 as $\sigma \to 0$ (DCG stands for Discounted Cumulative Gain \cite{jarvelin2002}).

\begin{proposition}
\label{prop:sdcg1}
The Lipschitz  constant of $\sdcg$ w.r.t. $\|\cdot\|_\infty$ satisfies $G_{\sdcg} \le 2D(1)G(Y_{\max})/\sigma$ for any $m \ge 1$. Here $Y_{\max}$ is maximum possible relevance score
of a document (usually less than 5).
\end{proposition}
As in the ListNet loss case we previously considered, the generalization bound resulting from Theorem~\ref{thm:lip} will be independent of $m$. This is intuitively satisfying: DCG@1, whose smoothing we are considering, only
depends on the document that is put in the top position by the score vector $s$ (and not on the entire sorted order of $s$). Our generalization bound does not deteriorate as the total list size $m$ grows. In contrast, the bound of~\citet{chapelle2010} will necessarily deteriorate as $\sqrt{m}$ since the constant $G_{\sdcg}^{CW}$ is the same as $G_{\sdcg}$. Moreover, it should be noted that even in the original SmoothedDCG paper, $\sigma$ is present in the denominator of $G_{\sdcg}^{CW}$, so our results are directly comparable. Also note that this example can easily be extended to consider DCG@$k$ for case when document list length $m \gg k$ (a very common scenario in practice).

\subsection{Application to RankSVM}
\label{sec:ranksvm}
RankSVM~\citep{joachims2002} is another well established ranking method, which minimizes a \emph{convex} surrogate based on pairwise comparisons of documents. A number of studies have shown that ListNet has better empirical performance than RankSVM. One possible reason for the better performance of ListNet over RankSVM is that the Lipschitz  constant of RankSVM surrogate w.r.t $\|\cdot\|_{\infty}$ doe scale with document list size as $O(m^2)$. Due to lack of space, we give the details in the supplement. 

%% file: conclusion.tex
We showed that it is not necessary for generalization error bounds for query-level learning to rank algorithms to 
deteriorate with increasing length of document lists associated with queries. The key idea behind our improved bounds
was defining Lipschitz constants w.r.t. $\ell_\infty$ norm instead of the ``standard" $\ell_2$ norm. As a result, we were able to derive much tighter guarantees
for popular loss functions such as ListNet and Smoothed DCG@1 than previously available.
%We also extended the basic generalization bounds to cover cases
%involving a) high-dimensional features and b) smooth loss functions.

Our generalization analysis of learning to rank algorithms paves the way for further interesting work. One possibility is
to use these bounds to design active learning algorithms for learning to rank with formal label complexity guarantees. Another interesting 
possibility is to consider other problems, such as multi-label learning, where functions with vector-valued outputs are learned by optimizing
a joint function of those outputs.